\documentclass[letterpaper, 10 pt, conference]{ieeeconf}  

\IEEEoverridecommandlockouts                              
\usepackage[bookmarks=true]{hyperref}
\usepackage{color}
\usepackage{times}
\usepackage{epsfig}
\usepackage{epstopdf}
\usepackage{graphicx}
\usepackage{amsmath}
\usepackage{amssymb}
\usepackage{bbm}
\usepackage[ruled,vlined,linesnumbered]{algorithm2e}
\usepackage{graphicx} 
\usepackage{subfigure} 
\usepackage{subfloat}
\usepackage{verbatim}
\usepackage{booktabs}
\usepackage{multirow}
\usepackage{makecell}
\usepackage{subfigure}
\usepackage{colortbl,booktabs}
\usepackage{lscape}
\usepackage{url}
\usepackage{diagbox}
\usepackage{makecell}
\usepackage{xcolor,colortbl}
\usepackage{float} 
\usepackage[english]{babel}

\newtheorem{theorem}{\textbf{Theorem}}
\newtheorem{lemma}{\textbf{Lemma}}

\newtheorem{corollary}{\textbf{Corollary}}
\usepackage{fixltx2e}
\usepackage{microtype}
\usepackage{wrapfig}

\title{\LARGE \bf
NAUTS: Negotiation for Adaptation to Unstructured Terrain Surfaces
}

\author{Sriram Siva$^{1}$, Maggie Wigness$^{2}$, John G. Rogers$^{2}$, Long Quang$^{2}$, and Hao Zhang$^{1,3}$
\thanks{*This work was partially supported by NSF CAREER Award IIS-1942056 and U.S. DEVCOM
ARL SARA Program W911NF-20-2-0107.}
\thanks{$^{1}$Sriram Siva is with the Computer Science Department, Colorado School of Mines (Mines), Golden, CO 80401, USA. Hao Zhang is partially affiliated with Mines.         {Email: sivasriram@mines.edu}. }%
\thanks{$^{2}$Maggie Wigness, John G. Rogers, and Long Quang are with the DEVCOM Army Research Laboratory (ARL), Adelphi, MD 20783, USA.  {Email: \{maggie.b.wigness.civ, john.g.rogers59.civ, long.p.quang.civ\}@army.mil}.
}
\thanks{$^{3}$Hao Zhang is with the Manning College of Information and Computer Sciences (CICS), 
University of Massachusetts Amherst, Amherst, MA 01002, USA.
        {Email: hao.zhang@cs.umass.edu}.}
}

\begin{document}

\maketitle
\thispagestyle{empty}
\pagestyle{empty}

\begin{abstract}
When robots operate in real-world off-road environments with unstructured  terrains, the ability to adapt their navigational policy is critical for effective and safe navigation. 
However, off-road terrains introduce several challenges to robot navigation, including dynamic obstacles and terrain uncertainty, leading to inefficient traversal or navigation failures. 
To address these challenges, we introduce a novel approach for adaptation by negotiation that enables a ground robot to adjust its navigational behaviors through a negotiation process. Our approach first learns prediction models for various navigational policies to function as a terrain-aware joint local controller and planner. 
Then, through a new negotiation process, our approach learns from various policies' interactions with the environment to agree on the optimal combination of policies in an online fashion to adapt robot navigation to unstructured off-road terrains on the fly. 
Additionally, we implement a new optimization algorithm that offers the optimal solution for robot negotiation in real-time during execution. 
Experimental results have validated that our method for adaptation by negotiation 
outperforms previous methods for robot navigation, especially over unseen and uncertain dynamic terrains. 
\color{black}
\end{abstract}

\section{Introduction}

In recent years, autonomous mobile robots have been increasingly deployed in off-road field environments to carry out tasks related to disaster response, infrastructure inspection, and subterranean and planetary exploration \cite{lattanzi2017review,schuster2019towards,chiang2020safety}. When operating in such environments, mobile robots encounter dynamic, unstructured terrains that offer a wide variety of challenges (as seen in Fig. \ref{motivation}), including dynamic obstacles and varying terrain characteristics like slope and softness. In these environments, terrain adaptation is an essential capability that allows ground robots to perform successful maneuvers by adjusting their navigational behaviors 
to best traverse
the changing unstructured off-road terrain characteristics \cite{sartoretti2018central,siva2019robot}.

Given its importance, the problem of robot adaptation over unstructured terrains has been extensively investigated in recent years.
In general, terrain adaptation has been addressed using three broad categories of methods.
The first category, classic control-based methods, use mathematical tools from control theory \cite{kavraki1996probabilistic,williams2016aggressive,moysis2020chaotic} to design robot models that achieve the desired robot behavior and perform robust ground maneuvers in various environments.
The second category, learning-based methods, use data-driven formulations
to either imitate an expert demonstrator \cite{siva2019robot,kahn2021land,kahn2021badgr}, learn from trial-and-error in a reinforcement learning setting \cite{kahn2018self,han2018sensor,kumar2021rma}, or use 
online learning 
to continuously learn and adapt in an environment 
\cite{liu2021lifelong, zenke2017continual, kahn2018composable}.
Finally, the third category, machine-learning-based control, exploits the advantage of integrating machine learning into control theory to learn accurate robot dynamics and accordingly adapt navigational behaviors
\cite{jhang2018navigation, sinha2020neural, sinha2020formulazero}.

\begin{figure}[t]
\centering
\includegraphics[width=0.49\textwidth]{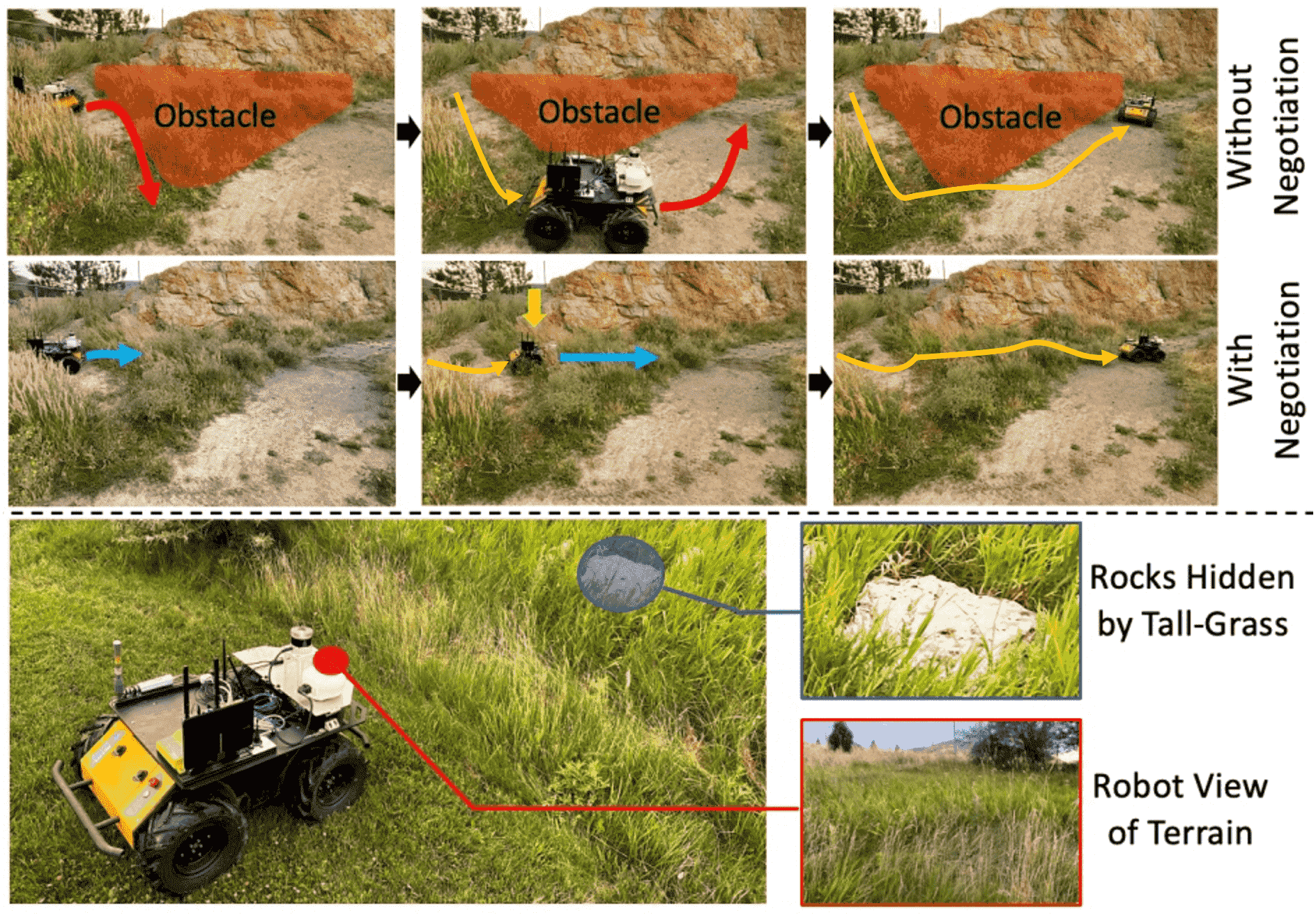}
\vspace{-11pt}
\caption{Robots operating in dynamic, unstructured environments often generate sub-optimal behaviors leading to inefficient robot traversal or even navigation failure. For example, robots may consider tall grass terrain as an obstacle.
Terrain negotiation allows robots to explore different navigation policies 
to determine the optimal combination for successful and efficient navigation in unknown terrains. 
In this example, the robot initially treats tall grass as an obstacle but simultaneously explores a max speed policy. The robot then quickly observes that the max speed policy improves efficiency by traversing across tall grass, and thus, learns to give more importance to the max speed policy compared to obstacle avoidance.}\label{motivation}
\end{figure}

However, unstructured terrains often have dynamic obstacles that change their state as the robot traverses over them, such as tall grass. 
Additionally, these terrains can occlude future obstacles and ground cover, leading to traversal uncertainty (e.g., grass occluding a rock as seen in Fig. \ref{motivation}).
These challenges can also be observed in commonly traversed unstructured environments such as sand, snow, 
mud, and forest terrains.
As characteristics of such terrains cannot be modeled beforehand, robots cannot be trained for all possible terrain variations and 
must therefore
adapt as these variations are encountered.
Existing methods for robot navigation generally
lack robustness to address these challenges as they are designed as a local controller to execute a single robot navigation policy, causing inefficient (e.g., longer traversal time and distance) or even failed navigation. 
In addition, current methods such as \cite{kahn2021land,kahn2021badgr} require significant amounts of training data to learn optimal navigational behaviors. 
The challenge of quickly learning a joint local controller and planner to enable adaptive behaviors has not been addressed. 

\color{blue}

\color{black}
In this paper, we introduce our novel approach to robot navigation: \emph{Negotiation for Adaptation to Unstructured Terrain Surfaces} (NAUTS).
Instead of generating terrain-aware behaviors for only the current time steps, NAUTS learns a non-linear prediction model to estimate future robot behaviors and 
states for several different policies. 
Each policy represents a series of navigational behaviors that can be learned either using imitation learning \cite{siva2019robot} or self-supervised learning \cite{kahn2021badgr} according to a specific goal (e.g., obstacle avoidance, maximum speed, etc.). 
NAUTS then learns from
the continuous interaction of these different policies 
with the terrain to generate optimal behaviors for successful and efficient navigation.
We define \emph{negotiation} as the process of learning robot navigation behaviors from online interactions between a library of policies with the terrain in order to agree on an optimal combination of these policies. 
The learning of both the non-linear prediction models and policy negotiation are integrated into a unified mathematical
formulation under 
a regularized optimization paradigm.

There are three main contributions of this paper:
\begin{itemize}
     \item We introduce a novel non-linear prediction model to estimate goal-driven future robot behaviors and states according to various navigational policies and address the challenge of learning a terrain-aware joint local controller and planner. 
     \item We propose one of the first formulations on negotiation for robot adaptation under a regularized optimization framework. Our approach allows a robot to continuously form agreements between various navigational policies and optimally combines them to i) improve the efficiency of navigation in known environments and ii) learn new navigation policies quickly in unknown and uncertain environments.
     \item We design a new optimization algorithm that allows for fast, real-time convergence to execute robot negotiation during deployment.
 \end{itemize}

As an experimental contribution, we provide a comprehensive performance evaluation of learning-based navigation methods over challenging dynamic unstructured terrains.


\section{Related Work}\label{sec:relatedwork}
The related research in robot terrain adaptation can be classified under methods based on classical control theory, learning-based, and machine-learning-based control.

The methods developed under the classical control theory use pre-defined models to generate robust navigational behaviors and reach the desired goal position in an outdoor field environment.
Earlier methods used a fuzzy logic implementation to perform navigation \cite{saffiotti1997uses, wang2008fuzzy}, without using the knowledge of a robot's dynamics.
This led to the development of system identification \cite{rabiner1978fir}, where methods learn robot dynamics using transfer functions to model linear robotic systems and perform navigation \cite{bolea2003non, pebrianti2018motion}. More recently, trajectory optimization models such as differential dynamic programming (DDP), specifically iterative linear quadratic regulator (iLQR), used knowledge of non-linear robot dynamics to solve navigation tasks  \cite{van2017motion,zhang2012iterative}.
Model predictive control (MPC) learns to be robust to robot model errors and terrain noise by implementing a closed-loop feedback system during terrain navigation \cite{howard2010receding,hafez2019integrity,tahirovic2010general}.  
However, these methods can approximate robot dynamics to a limited extent as these methods cannot learn from high-dimensional robot data and lack the ability to adapt as terrain changes.

Learning-based methods use data-driven formulations to generate navigational behaviors in various environments. 
Early methods used Koopman operator theory \cite{koopman1931hamiltonian} to model non-linear robot systems using an infinite-dimensional robot observable space \cite{proctor2018generalizing, williams2015data} to perform terrain navigation.  
Subsequent learning-based methods mainly used learning from demonstration (LfD) \cite{atkeson1997robot} to transfer human expertise of robot driving to mobile robots \cite{kahn2021land, wigness2018robot}. 
One method to perform terrain-aware navigation combined representation learning for terrain classification with apprenticeship learning to perform terrain adaptation \cite{siva2019robot}. 
Kahn and Levine \cite{kahn2021badgr} learned navigational affordance from experts over various terrains for carrying out off-road navigation.
Recently, consistent behavior generation was achieved \cite{siva2021enhancing} to match actuation behaviors with a robot's expected behaviors. 
Unlike learning from demonstration, reinforcement learning based methods learn purely from a robot's own experience in an unknown environment in a trial-and-error fashion \cite{kahn2018self,han2018sensor}. Rapid motor adaptation was achieved by updating learned policies via inferring key environmental parameters to successfully adapt in various terrains \cite{kumar2021rma}.
Life-long learning methods, similar to reinforcement learning, sequentially improve the performance of robot navigation by continuously optimizing learned models \cite{kahn2018composable,wang2021appli}.
Rather than just learning a robot model, learning-based methods also learn robot interactions with the terrain, thus being terrain-aware.  
However, these methods fail in unstructured environments \cite{nampoothiri2021recent} as they cannot adapt on the fly with the terrain or 
exhibit
catastrophic forgetting \cite{serra2018overcoming}, which is the tendency to forget previously learned data upon learning from new data.

\begin{figure*}[th]
\centering
\includegraphics[width=0.98\textwidth]{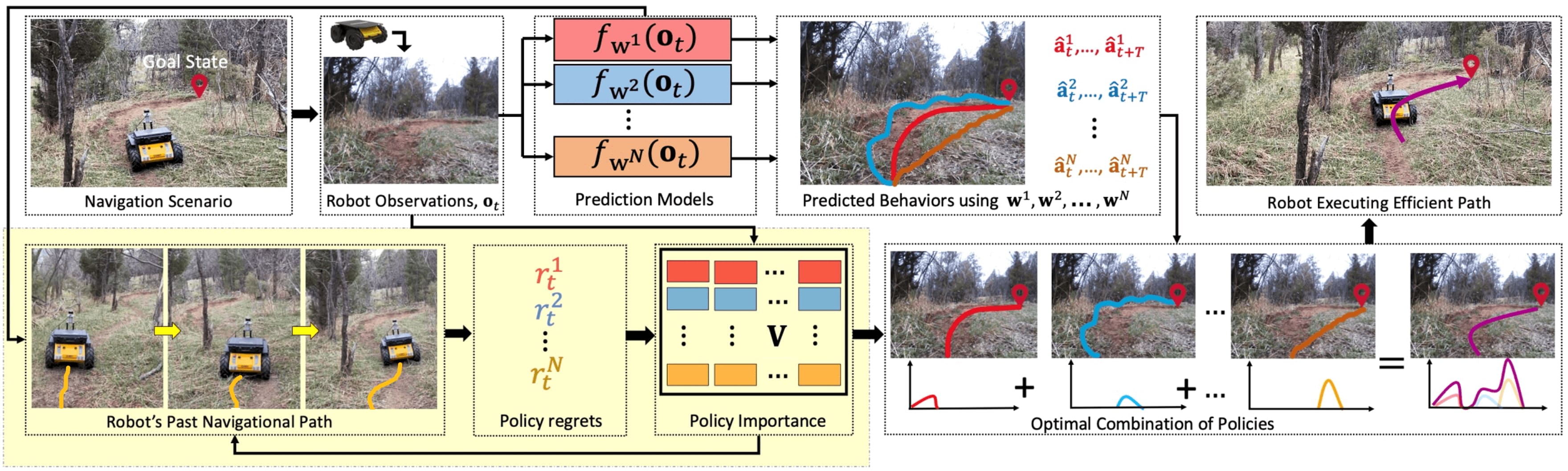}
\vspace{-3pt}
\caption{Overview of our proposed NAUTS approach for robot negotiation to adapt over unstructured terrains. Illustrated is the learning performed by our approach during the training phase. The module in the yellow box illustrates robot negotiation during the execution stage.
}\label{approach}
\vspace{-6pt}
\end{figure*}

Machine-learning-based control methods learn robot behaviors by combining data-driven formulations into predefined robot models \cite{duriez2017machine,brunton2019data}.
Early methods used Dynamics Mode Decomposition (DMD) \cite{schmid2010dynamic} and Sparse Identification of Non-Linear Dynamics (SINDy) \cite{mamakoukas2019local} to learn data-driven models based on system identification and performed terrain navigation \cite{wang2021real,kutz2016dynamic}.
Later, evolutionary algorithms were developed to optimize parameters of a robot model in an online learning fashion for robust navigation \cite{caceres2017approach,ramirez1999nonlinear}. For robots with multiple degrees of freedom, methods were developed that use a combination of iterative Linear Quadratic Regulators (iLQR) and machine learning search to explore multiple robot configurations and plan self-adaptive navigation \cite{gillespie2018learning}. Similar approaches were designed using a neural network based functional approximator to learn a robot dynamics model and adapt this model with online learning \cite{nagariya2020iterative}. Robust path planning was performed for safe navigation of autonomous vehicles under perception uncertainty \cite{alharbi2020global}. 
However, these methods do not address adaptation to previously unseen, unstructured terrains, and 
are unable
to address the dynamic nature of the terrain, which often leads to ineffective terrain traversal.

\section{Approach}\label{sec:Approach}

In this section, we discuss our proposed method, NAUTS, for robot traversal adaptation by negotiation. An overview of the approach is illustrated in Fig. \ref{approach}.

\subsection{Learning Policy Prediction Models}

Our approach first learns a non-linear prediction model to estimate future robot states and behaviors 
for each policy in a previously trained library. Navigational policies describe various goals of navigation, e.g., obstacle avoidance, adaptive maneuvers or max speed.
This model enables our approach to predict how a policy works
without the requirement of knowing its implementation (i.e., the policy can be treated as a black
box). 
Formally, at time $t$, we denote the robot terrain observations (e.g., RGB images) input to the $i$-th
policy as $\mathbf{o}^{i}_{t}\in \mathbb{R}^{q}$, where $q$ is the dimensionality of the terrain observations. The robot behavior controls, i.e,  navigational behaviors (e.g., linear and angular velocity), and states (e.g., robot's body pose and position) output from the policy are denoted as $\mathbf{a}^{i}_{t} \in \mathbb{R}^{c}$ and $\mathbf{s}^{i}_{t} \in \mathbb{R}^{m}$, with $c$ and $m$ denote the dimensionality of robot behaviors and states respectively.
Then the $i$-th policy can be represented as $\pi^{i}: (\mathbf{s}^{i}_{t},\mathbf{o}^{i}_{t}) \rightarrow \mathbf{a}_{t}^{i}$.

\begin{figure}[b]
\vspace{-9pt}
\centering
\includegraphics[width=0.42\textwidth]{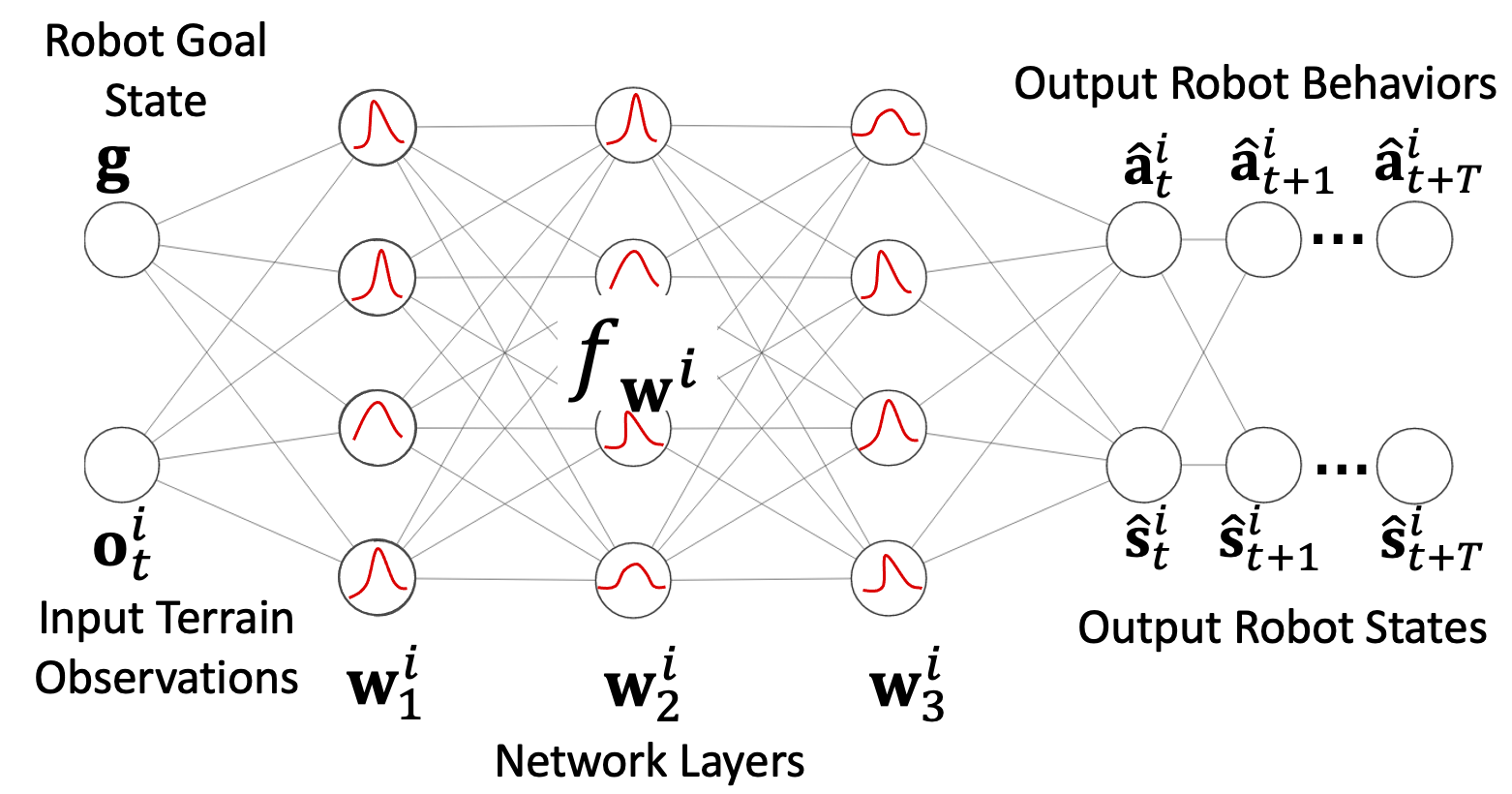}
\vspace{-3pt}
\caption{A shallow GP is designed to implement our prediction model $f_{\mathbf{w}^{i}}$.
}\label{SGP}
\end{figure}

Let $\mathbf{g}$ denote the relative goal state (with respect to $\mathbf{s}_{t}^{i}$) that the robot needs to reach at time $t+T$. 
For every policy $\pi^{i}$, we propose to learn a prediction model $f_{\mathbf{w}^{i}}:(\mathbf{o}_{t}^{i},\mathbf{g}) \rightarrow (\mathbf{\hat{a}}^{i}_{t:t+T},\mathbf{\hat{s}}^{i}_{t:t+T})$ that predicts a sequence of goal driven $T$-future robot behaviors $\mathbf{\hat{a}}^{i}_{t:t+T}$ and states $\mathbf{\hat{s}}^{i}_{t:t+T}$. 
The prediction model estimates behaviors for the present time and functions like a local controller, and by estimating robot behaviors and states for future $T$-steps, it functions as a local planner.
We introduce a shallow Gaussian Process (GP) \cite{garriga2018deep} to implement $f_{\mathbf{w}^{i}}$ that is parameterized by $\mathbf{w}^{i}$, as shown in Fig. \ref{SGP}. 
This shallow Gaussian Process with a recursive kernel has been shown in \cite{garriga2018deep} to be equivalent to, but more data-efficient than, a deep Bayesian CNN with infinitely many filters.
In addition, as this Gaussian Process assumes that each weight of the network is a distribution instead of 
scalar
values, it allows for uncertainty modeling and thus, is robust to environmental variations. We then learn the prediction model $f_{\mathbf{w}^{i}}$ by solving the following regularized optimization problem:
\begin{align}\label{eqn1}
\min_{\mathbf{w}^{i}}  \:\:\:
&\lambda_{1}  \mathcal{L} \big(( \pi^{i}(\mathbf{s}^{i}_{t:t+T},\mathbf{o}^{i}_{t:t+T}), \mathbf{s}^{i}_{t:t+T}),f_{\mathbf{w}^{i}}(\mathbf{o}^{i}_{t},\mathbf{g})\big) \nonumber \\ 
& \:\: + \lambda_{2} \Vert \mathbf{g}^{i}-(\mathbf{\hat{s}}^{i}_{t+T}-\mathbf{\hat{s}}^{i}_{t}) \Vert_{2}^{2}
\end{align}
where $\mathcal{L}(\cdot)$ is the cross-entropy loss \cite{zhang2018generalized}, mathematically expressed as  $\mathcal{L}(p,q) = -\mathbb{E}_{p}[\log(q)]$. This loss helps the prediction model to be insensitive to noisy observations in unstructured environments due to the logarithmic scale. The first part of Eq. (\ref{eqn1}) models the error of predicting $T$-future robot behaviors and states from actual navigational behaviors and states. The second part of Eq. (\ref{eqn1}) models the error of the robot failing to reach its relative goal state. The hyper-parameters $\lambda_{1}$ and $\lambda_{2}$ model the trade-off between the losses.

Following Eq. ({\ref{eqn1}}), the robot learns prediction models for $N$-different policies. 
However, when navigating over unstructured terrains, a single policy may not always prove to be effective for all scenarios.
For example, the policy of obstacle avoidance may lead to longer traversal time in grass terrain, and the policy of max speed may cause collisions with occluded obstacles. 

\subsection{Robot Negotiation for Terrain Adaptation}
The key novelty of NAUTS is its capability of negotiating between different policies to perform successful and efficient navigation, especially in unstructured off-road terrains.
Given $N$-policies in the library, NAUTS formulates robot adaptation by negotiation under the mathematical framework of multi-arm bandit (MAB) optimization \cite{chan2019assistive}. 
MAB comes from the hypothetical experiment where the robot must choose between multiple policies, each of which has an unknown regret with the goal of determining the best (or least regretted) outcome on the fly. We define regret, $r^{i}_{t} : (\mathbf{o}^{i}_{t-T},\mathbf{g})\big) \rightarrow \mathbb{R}^{+}$, of the $i$-th policy at time $t$ as the error of not reaching i) the goal position and ii) the goal position in minimum time and effort. We calculate the regret for each policy as:
\begin{align}\label{eqn2}
r^{i}_{t} = \Big(\frac{\Vert \mathbf{g}\Vert_{2}\Vert \mathbf{\hat{s}}^{i}_{t}\Vert_{2}}{(\mathbf{g})^\top( \mathbf{\hat{s}}^{i}_{t})}-1\Big) 
+\sum_{k=t-T}^{t}(t-k) (\mathbf{\hat{a}}^{i}_{k})^{\top}\mathbf{\hat{a}}^{i}_{k}
\end{align}
where the first part of Eq. (\ref{eqn2}) models the error of not reaching the goal position,
with zero regret if the robot reached its goal position. This error grows exponentially if the robot has deviated from the goal position. 
The second part of Eq. (\ref{eqn2}) models the error of not reaching the goal in minimum time and effort. 
Specifically, the regret is smaller when the robot uses fewer values of navigational behaviors to reach the same goal and also if the robot reaches the goal in minimum time due to the scaling term $(t-k)$. 

Unstructured terrain-aware negotiation can be achieved using the best subset of policies that minimize the overall regret in the present terrain as:
\begin{align}\label{eqn2.5}
    &\min_{\mathbf{V}} \quad \lambda_{3}\sum_{i=1}^{N} \mathcal{R}(\mathbf{o}_{t}^{i},{r}^{i}_{t};\mathbf{v}^{i}) + \lambda_{4}\Vert\mathbf{V}\Vert_{E} \\
&\textrm{s.t.} \quad \quad \quad \sum_{i=1}^{N}(\mathbf{o}^{i}_{t})^{\top}\mathbf{v}^{i}=1  \nonumber
\end{align}
where $\mathcal{R}(\cdot)$, parameterized by $\mathbf{v}^{i} \in \mathbb{R}^{q}$,  is the terrain-aware regret of choosing 
policy $\pi^{i}$ in the present terrain and $\mathbf{V}=[\mathbf{v}^{1},\dots,\mathbf{v}^{N}] \in \mathbb{R}^{N \times q}$. Mathematically, $\mathcal{R}(\mathbf{o}_{t}^{i},{r}^{i}_{t};\mathbf{v}^{i}) = \sum_{k=t}^{t+T} \Vert r^{*}_{k} - (\mathbf{o}_{t}^{i})^{\top}\mathbf{v}^{i} {r^{i}_{k}} \Vert^{2}_{2}$, with $r^{*}_{k} = \min r_{k}^{i}; i=1,\dots,N$. The use of a linear model enables real-time convergence for terrain-aware policy negotiation. 
The column sum of $\mathbf{V}$ 
indicates 
the weights of each policy towards minimizing the overall regret of robot navigation. In doing so, the robot recognizes the important policies and exploits these policies to maintain efficient navigation. However, we also need to explore the various policies to improve navigation efficiency or even learn in an unknown environment, which is achieved by the 
regularization term in Eq. (\ref{eqn2.5}), called the exploration norm. Mathematically, $\Vert\mathbf{V} \Vert_{E} = \sum_{i=1}^{N} \frac {\Vert \mathbf{V} \Vert_{F}}{\Vert \mathbf{v}^{i} \Vert_{2}}$, where the operator $\Vert \cdot \Vert_{F}$ is the Frobenius norm
with $\Vert \mathbf{V} \Vert_{F} = \sqrt{\sum_{i=1}^{N}\sum_{j=1}^{q}(v^{i}_{j})^{2}}$.
The exploration norm enables NAUTS to continuously explore all navigational policies in any terrain. Specifically, the exploration norm enables NAUTS to explore 
sub-optimal policies
by ensuring $\mathbf{v}^{i} \neq \mathbf{0}$. If $\mathbf{v}^{i} = \mathbf{0}$, i.e., if the $i$-th policy is given zero importance, then the value of objective in Eq. (\ref{eqn2.5}) would be very high. The hyper-parameters $\lambda_{3}$ and $\lambda_{4}$ control the trade-off between exploration and exploitation during negotiation.
The constraints in Eq. (\ref{eqn2.5}) normalize the various combination of navigational policies.

Integrating prediction model learning and policy negotiation under a unified mathematical framework, robot adaptation by negotiation can be formulated as the following regularized optimization problem:
\begin{align}\label{eqn3}
&\min_{\mathbf{W},\mathbf{V}}  \sum_{i=1}^{N}\Big(
\lambda_{1}  \mathcal{L} \big(( \pi^{i}(\mathbf{s}^{i}_{t:t+T},\mathbf{o}^{i}_{t:t+T}), \mathbf{s}^{i}_{t:t+T}),f_{\mathbf{w}^{i}}(\mathbf{o}^{i}_{t},\mathbf{g})\big) \nonumber \\ 
&+ \lambda_{2} \Vert \mathbf{g}^{i}-(\mathbf{\hat{s}}^{i}_{t+T}-\mathbf{\hat{s}}^{i}_{t}) \Vert_{2}^{2} +  \lambda_{3} \mathcal{R}(\mathbf{o}_{t}^{i},{r}^{i}_{t};\mathbf{v}^{i})\Big) + \lambda_{4}\Vert\mathbf{V}\Vert_{E} \nonumber \\
&\textrm{s.t.} \quad \quad \quad \sum_{i=1}^{N}(\mathbf{o}^{i}_{t})^{\top}\mathbf{v}^{i}=1 
\end{align}
where $\mathbf{W}=[\mathbf{w}^{1},\dots,\mathbf{w}^{N}]$. 
During the training phase, we compute the optimal $\mathbf{W}^{*}$ and $\mathbf{V}^{*}$. 

During execution, we fix $\mathbf{W}^{*}$, meaning the prediction models do not update during execution. However, our approach continuously updates $\mathbf{V}^{*}$ in an online fashion, which allows for negotiation at each step. 
At every time step $t_0$, we acquire observations $\mathbf{o}_{t_{0}}$.
For a given robot goal state $\mathbf{g}$, 
we dynamically choose the best combination of policies as:
\begin{equation}\label{eqn4}
    \mathbf{a}_{t_{0}:t_{0}+T} =  \sum_{i=1}^{N} (\mathbf{o}_{t_0})^{\top}\mathbf{v}^{i*}f_{\mathbf{w}^{i*}}(\mathbf{o}_{t_{0}},\mathbf{g})
\end{equation}
where $\mathbf{a}_{t_{0}}$ is the behavior executed by the robot following policy negotiation at time $t_{0}$ and the behaviors $\mathbf{a}_{t_{0}:t_{0}+T}$ make up the local plan for the robot.

\begin{algorithm}[t]
    \SetKwInOut{Input}{Input}
    \SetKwInOut{Output}{Output}
    \SetKwInOut{return}{return}
\small
    \Input{Policies $\mathbf{W}^{*}$ and Weights $\mathbf{V}^{*} \in \mathbb{R}^{N \times q}$}

    \Output{Optimized Weights for Negotiation $\mathbf{V}^{*} \in \mathbb{R}^{N \times q}$ }

    \While{goal is not reached}
    {
    \For{$i=1,\dots,N$}
    {
        Obtain predicted behavior $\mathbf{\hat{a}}^{i}_{t:t+T}$ and states $\mathbf{\hat{s}}^{i}_{t:t+T}$ from $f_{\mathbf{w}^{i*}}(\mathbf{o}_{t_{0}},\mathbf{g})$;

        Calculate regret of $i$-th policy $r^{i}$ from Eq. (\ref{eqn2});
    }

     Calculate $r^{*}_{t_{0}} = \min r^{i}_{t_{0}}; \:\: i=1,\dots,N$;

    \While{not converge}
    {
     Calculate diagonal matrix $\mathbf{Q}$ with the $i$-th diagonal block given as $\frac{\mathbf{I}}{2\Vert \mathbf{V} \Vert_{E}}$;
     
     Compute the columns of the distribution $\mathbf{V}$ according to Eq. (\ref{optW2});
    }
    }
    \textbf{return: } $\mathbf{V}^{*} \in \mathbb{R}^{N \times q}$
 \caption{Optimization algorithm for solving the robot negotiation problem during execution in Eq. (\ref{eqn2.5}).} 
 \label{alg2}
\end{algorithm}

\subsection{Optimization Algorithm}
During training, we reduce Eq. (\ref{eqn3}) to simultaneously optimize $\mathbf{W}^{*}$ and $\mathbf{V}^{*}$. As the first term is non-linear,  reducing Eq. (\ref{eqn3}) amounts to optimizing a non-linear objective function. We use the zeroth order non-convex stochastic optimizer from \cite{balasubramanian2021zeroth}.
This optimizer has been proven to avoid saddle points and avoids local minima during optimization \cite{balasubramanian2021zeroth}, and is specifically designed for constrained optimization problems like in Eq. (\ref{eqn3}). Additionally due to its weaker dependence on input data dimensionality \cite{balasubramanian2021zeroth}, $\mathbf{W}$ and $\mathbf{V}$ can be computed faster despite using high dimensional terrain observations.

To perform robot adaptation by negotiation, we optimize $\mathbf{V}$ in an online fashion during the execution phase by solving the MAB optimization problem in Eq. (\ref{eqn2.5}), which has a convex objective with non-smooth regularization term. To perform fast online learning for negotiation, we introduce a novel iterative optimization algorithm that is tailored to solve the regularized optimization in Eq. (\ref{eqn2.5}), which at each time step performs fast iterations and converges in real-time to a global optimal value of $\mathbf{V}$. This optimization algorithm is provided in Alg. \ref{alg2}. Specifically, to solve for the optimal $\mathbf{V}$, we minimize Eq. (\ref{eqn2.5}) with respect to $\mathbf{v}^{i}$, resulting in:
 \begin{equation}\label{optW1}
  \sum_{k=t}^{t+T}\lambda_{3}\big(2(r^i_{k})^2(\mathbf{o}^{i}_{t})^{\top}(\mathbf{o}^{i}_{t})\mathbf{v}^{i} - 2{r^*_{k}}{r^{i}_{k}}\mathbf{o}^{i}_{t}\big) + \lambda_{4}\mathbf{Q}\mathbf{v}^{i} = 0
 \end{equation}
where $\mathbf{Q}$ is a block diagonal matrix expressed as $\mathbf{Q} = \frac{\mathbf{I}}{2\Vert \mathbf{V}\Vert_{E}}$ and $\mathbf{I}\in \mathbb{R}^{N \times N}$ is an identity matrix.
Then, we compute $\mathbf{v}^{i}$ in a closed-form solution as:
 \begin{equation}\label{optW2}
\mathbf{v}^{i} =(\lambda_{4}\mathbf{Q} + 2\sum_{k=t}^{t+T}\lambda_{3}(r^i_{k})^2(\mathbf{o}^{i})^{\top}\mathbf{o}^{i})^{-1} \lambda_{3}\sum_{k=t}^{t+T} ( 2{r^*_{k}}{r^{i}_{k}}\mathbf{o}^{i})
 \end{equation}
Because $\mathbf{Q}$ and $\mathbf{V}$ are
interdependent, 
we are able to derive
an iterative algorithm to compute them as described in Algorithm \ref{alg2}.

\noindent \textbf{Convergence.}
Algorithm \ref{alg2} is guaranteed to converge to the optimal solution for the optimization problem in Eq. (\ref{eqn2.5})\footnote{\label{supplementary}Derivation is provided at the end of the document}.

\noindent \textbf{Complexity.} 
For each iteration of Algorithm \ref{alg2}, computing Steps 3, 4, and 7 is trivial, and  Step 8 is computed by solving a system of linear equations with quadratic complexity.

\section{Experiments}\label{sec:experiments}
This section presents the experimental setup and implementation details of our NAUTS approach, and provides a comparison of our approach with several previous state-of-the-art methods.

\subsection{Experimental Setup}

We use a Clearpath Husky ground robot for our field experiments. The robot is equipped with an Intel Realsense D435 color camera, an Ouster OS1-64 LiDAR, a Global Positioning System (GPS), and an array of sensors including a Microstrain 3DM-GX5-25 Inertial Measurement Unit (IMU) and wheel odometers.
The robot states, i.e., robot pose, are estimated using an Extended Kalman Filter (EKF) \cite{rigatos2010extended}, applied on sensory observations from LiDAR, IMU, GPS, and wheel odometers.
The RGB images and the estimated robot states are used as our inputs.
The robot runs a 4.3 GHz i7 CPU with 16GB RAM and Nvidia 1660Ti GPU with 6GB of VRAM, which runs non-linear behavior prediction models at 5Hz and policy negotiation at ~0.25 Hz.

We evaluate our approach on navigation tasks that require traversing from the robot's initial position to a goal position, and provide a performance comparison against state-of-the-art robot navigation techniques
including Model Predictive Path Integral (MPPI) \cite{williams2016aggressive} control, Terrain Representation and Apprenticeship Learning (TRAL) \cite{siva2019robot}, Berkley Autonomous Driving Ground Robot (BADGR) \cite{kahn2021badgr}, and Learning to Navigate from Disengagements (LaND) \cite{kahn2021land}. To quantitatively evaluate and compare these approaches to NAUTS, we use the following evaluation metrics:
\begin{itemize}
\item \emph{Failure Rate (FR)}: This metric is defined as the number of times the robot fails to complete the navigation task across a set of experimental trials.
If a robot flips or is stopped by a terrain obstacle,
it is considered a failure.
Lower values of FR indicate better performance.
\item \emph{Traversal Time (TT)}: This metric is defined as the time taken to complete the navigation task over given terrain.
Smaller values of TT indicate better performance.
\item \emph{Distance traveled (DT)}: This metric is defined as the total distance traveled by the robot when completing a navigational task. A smaller DT value may indicate better performance.
\item \emph{Adaptation time (AT)}: This metric is defined as the time taken by the robot to regain half its linear velocity when introduced to an unseen unstructured environment. A lower value of AT may indicate better performance.
\end{itemize}

\begin{figure*}[t]
  \subfigure[Tall grass terrain]{
    \label{tall_grass_scenario}
    \begin{minipage}[b]{0.33\textwidth}
      \centering
        \includegraphics[width=0.98\textwidth]{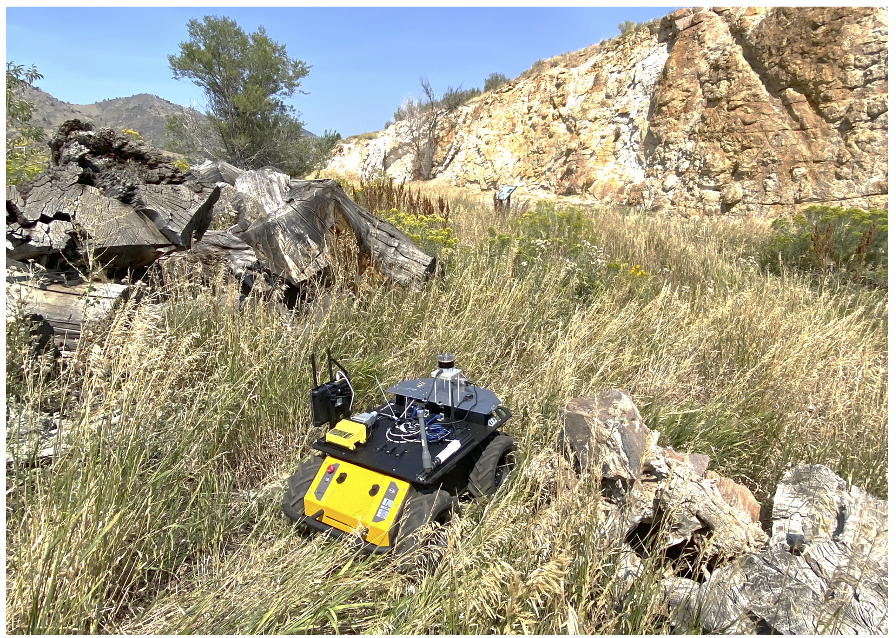}
        \centering
    \end{minipage}}
   \subfigure[Paths taken by different methods]{
    \label{tall_grass_paths}
    \begin{minipage}[b]{0.33\textwidth}
      \centering
        \includegraphics[width=0.98\textwidth]{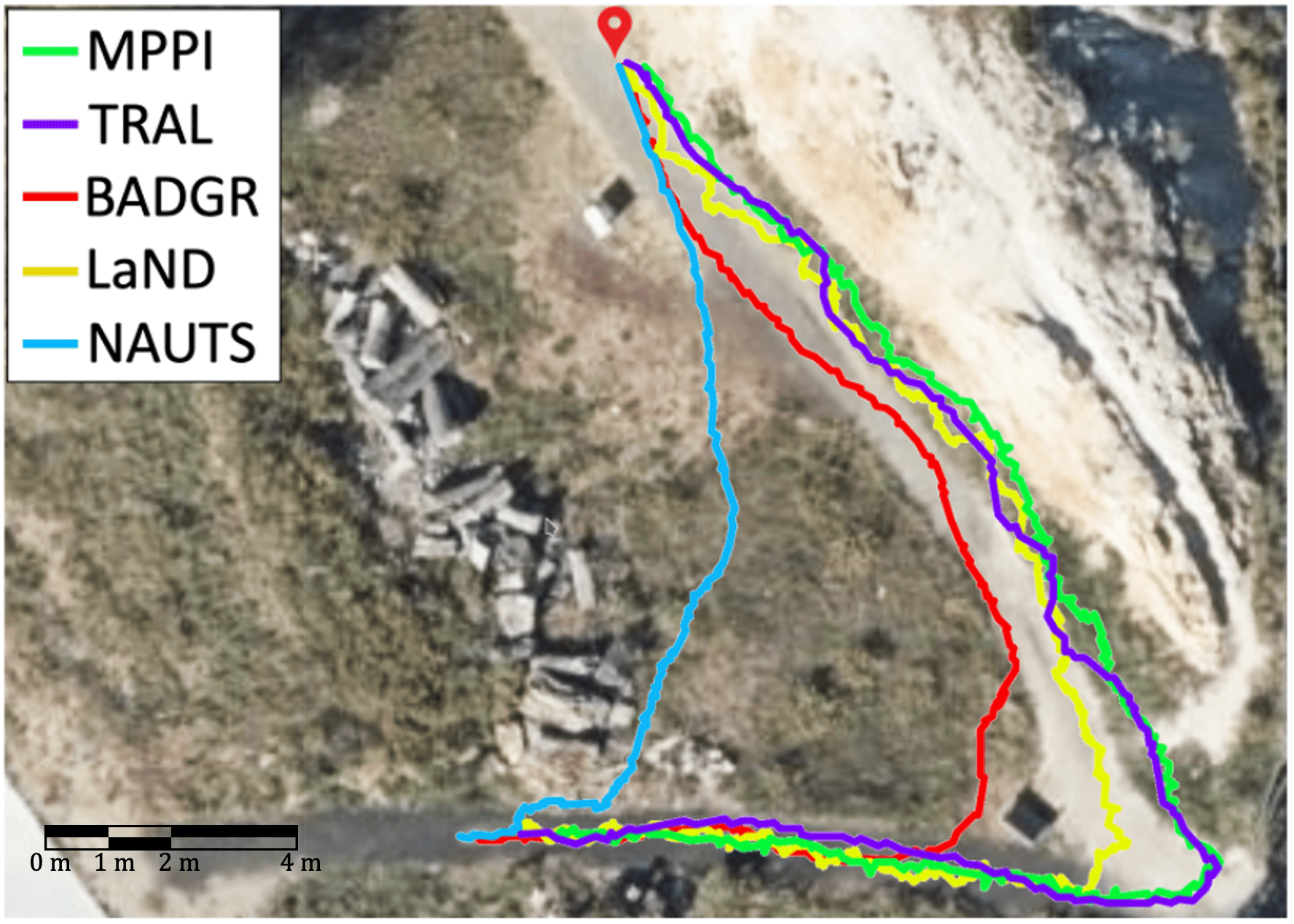}
        \centering
    \end{minipage}}
    \subfigure[Importance of different policies]{
    \label{tall_grass_negotiation}
    \begin{minipage}[b]{0.33\textwidth}
      \centering
        \includegraphics[width=0.87\textwidth]{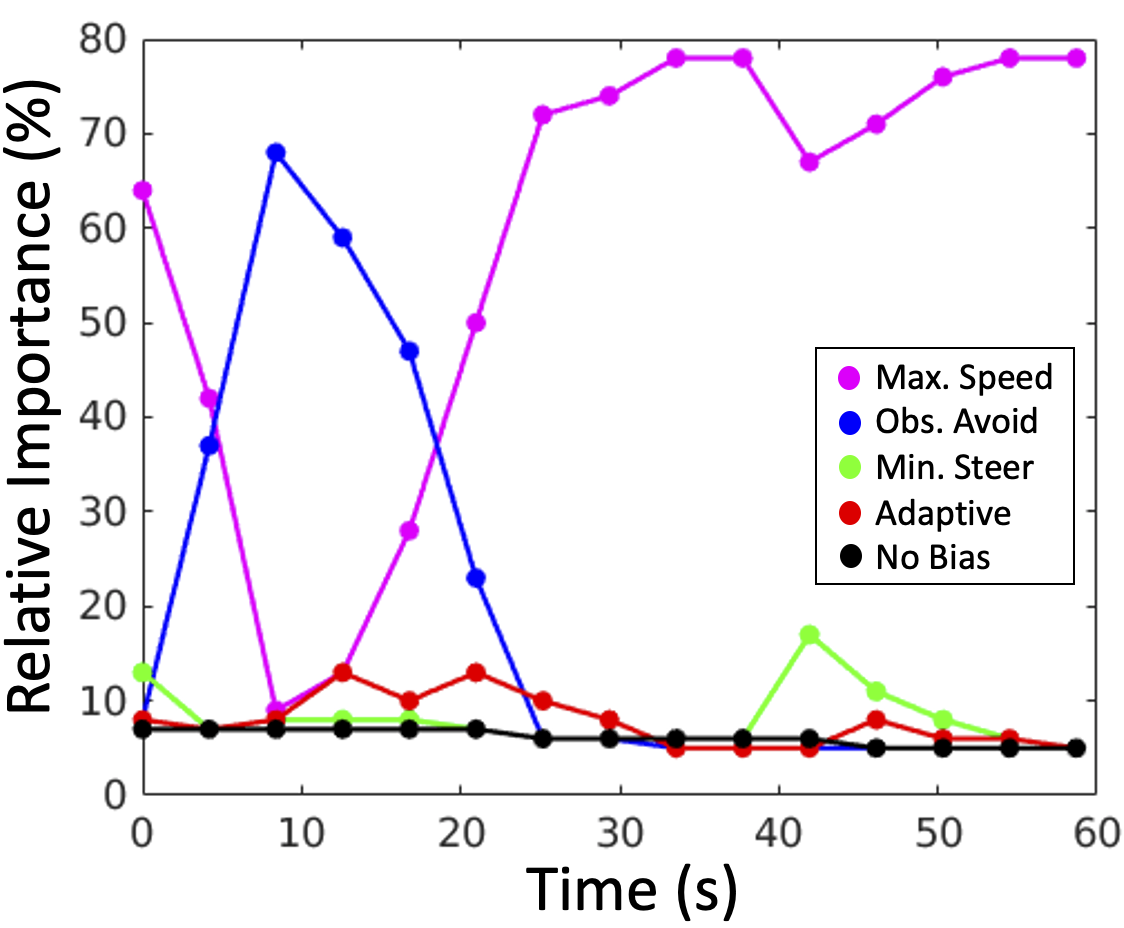}
        \centering
    \end{minipage}
      }
    \vspace{-8pt}
  \caption{The tall grass terrain used in our experiments and the qualitative results over this terrain.
  }  \label{tall_grass_analysis} 
\vspace{-6pt}
\end{figure*}

To collect the training data, a human expert demonstrates robot driving over simple terrains of concrete, short grass, gravel, medium-sized rocks, large-sized rocks and forest terrain. 
Each of these terrain were used to learn one specific aspect of robot navigation such as adjusting traversal speeds over large-sized rocks, or obstacle avoidance using the forest terrain. 
Specifically, we used these terrains to learn 
from a library of five 
distinct navigational policies:

\begin{itemize}
    \item \emph{Maximum Speed:} When following this navigational policy, the human expert drives with the maximum traversal speed irrespective of the terrain the robot traverses upon. The aim when following the maximum speed navigational policy is to teach the robot to cover as much distance as possible in the least amount of time. Thus, while collecting training data with this policy the expert demonstrator uses straight line traversal without steering the robot. 
    \item \emph{Obstacle Avoidance:} While following this policy, the expert demonstrates how to maneuver by driving around obstacles to avoid collision. To learn this policy, expert demonstrations in forest terrains are used where humans navigate the forest by avoiding trees and logs
    while moving the robot through the terrain.
    The underlying goal with this policy is to teach the robot to steer around obstacles.  
    \item \emph{Minimum Steering:} For this policy, the expert drives the robot with limited steering. 
    During navigation, linear velocity is fixed to 0.75 m/s and obstacle avoidance is performed by beginning to turn the robot when it is further away from obstacles instead of making short, acute turns. The policy differs from obstacle avoidance by maintaining a fixed speed while taking a smooth and long maneuver around obstacles.
    \item \emph{Adaptive Maneuvers:}  While following this policy, the expert varies the robot's speed across different terrain to reduce traversal bumpiness. Specifically, with terrains that are relatively less rugged such as concrete or short-grass, the expert demonstrator uses high speed maneuvers. On the other hand, over terrains with high ruggedness such as gravel or medium sized rocks, the expert demonstrator uses slower speeds, with the slowest traversal speed across the large rocks terrain.
    \item \emph{No Navigational Bias:} When following this policy, multiple expert demonstrators navigate the robot across the different terrains without particular policy bias, i.e., without following any specific navigational policy. The underlying goal behind using such policy is to cover most of the common navigational scenarios encountered by the robot, and include the navigational bias from multiple expert demonstrators. 
\end{itemize}

For each policy, the robot is driven on each of the different terrains, resulting in approximately 108000 distinctive terrain observations with the corresponding sequence of robot navigational behaviors and states for each navigational policy.
No further pre-processing is performed on the collected data. We use this data to learn optimal $\pi^{i}$, $i=1,\dots, N$ and $\mathbf{V}$ during training. We learn these parameters for different values of hyper-parameters of the NAUTS approach, i.e., $\lambda_{1}$, $\lambda_{2}$, $\lambda_{3}$, $\lambda_{4}$ and $T$. 
The combination of these hyper-parameters that results in the best performance of NAUTS during 
validation
are then used for our experiments. 
In our case, the optimal performance of NAUTS is obtained at  
 $\lambda_{1}=0.1, \lambda_{2}=10$, $\lambda_{3}=1$ and $\lambda_{4}=0.1$ for $T=9$.

\begin{table}[h]
\centering
\caption{Quantitative results for scenarios when the robot traverses over
dynamic, uncertain grass terrain. 
}
\label{tab:RTT1}
\tabcolsep=0.07cm
\centering
\begin{tabular}{ | c|| c| c| c| c| c|}
\hline
Metrics & MPPI \cite{williams2016aggressive}  & TRAL \cite{siva2019robot} & BADGR \cite{kahn2021badgr}  &  LaND \cite{kahn2021land} & \textbf{NAUTS} \\
\hline
FR (/10) & 3 & 3 & \textbf{1} & 5 & \textbf{1} \\
TT (s) & 88.72 & 72.99 & 64.47 & 90.18 & \textbf{58.79} \\
DT (m) & 68.58 & 56.69 & 50.29 & 64.93 & \textbf{36.57} \\
AT (s) & 14.23 & 10.92 & -- & -- & \textbf{6.24} \\
\hline
\end{tabular}
\vspace{-6pt}
\end{table}

\subsection{Navigating over Dynamic Uncertain Grass Terrain}

In this set of experiments, we evaluate robot traversal performance over the tall grass terrain environment, as shown in Fig. \ref{tall_grass_scenario}. This is one of the most commonly found terrains in off-road environments and is characterized by deformable dynamic obstacles added with the terrain uncertainty of occluded obstacles. 
The process of negotiation is continuously performed throughout the execution phase. 
The evaluation metrics for each of the methods are computed across ten trial runs over the tall grass terrain environment.

The quantitative results obtained by our approach and its comparison with other methods are presented in Table \ref{tab:RTT1}. 
In terms of the FR metric, BADGR and NAUTS obtain the lowest values, whereas MPPI, TRAL and LaND have high FR values.
Navigation failure for MPPI, TRAL and LaND generally occurred as the robot transitioned into the tall grass terrain where it would get stuck after determining the tall grass was an obstacle. Failure cases for NAUTS and BADGR occurred when the robot was stuck in the tall grass terrain due to wheel slip. 
Both NAUTS and BADGR obtain significantly fewer failures than MPPI and LaND methods due to their ability to adapt to different terrains.

\begin{figure*}[t]
  \subfigure[Forest terrain]{
    \label{forest_scenario} 
    \begin{minipage}[b]{0.33\textwidth}
      \centering
        \includegraphics[width=0.98\textwidth]{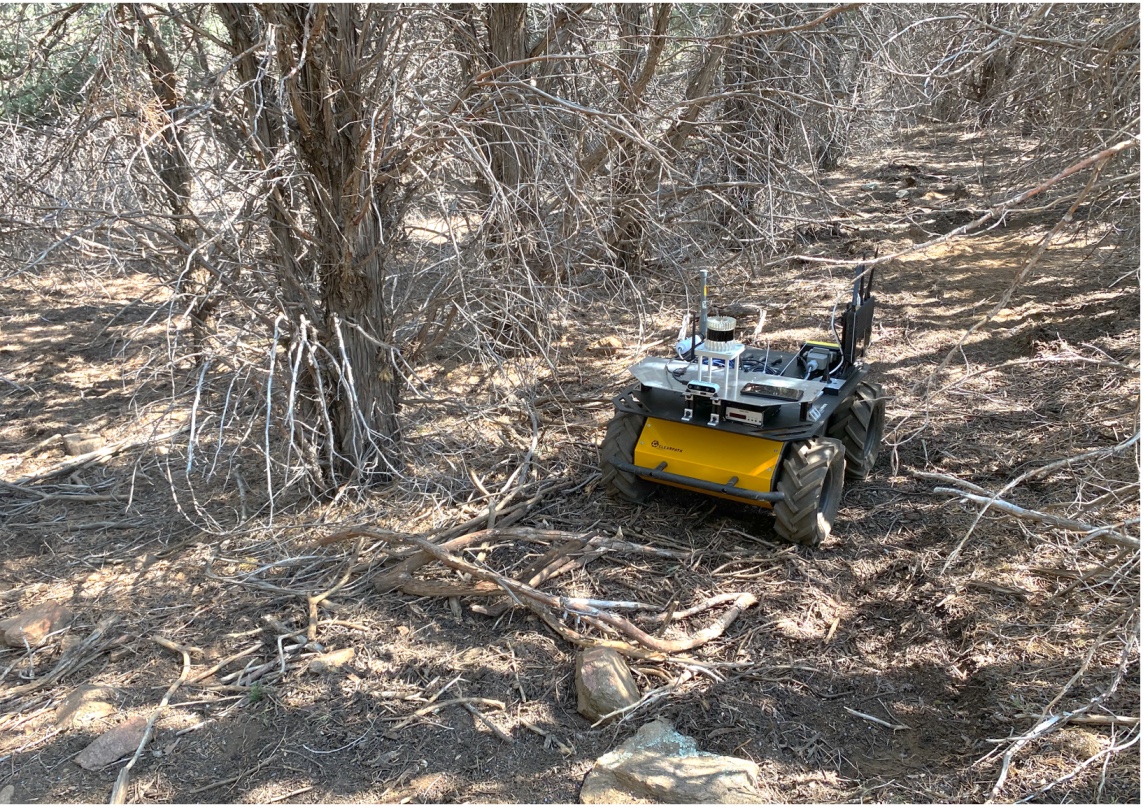}
        \centering
    \end{minipage}}
   \subfigure[Path taken by different methods]{
    \label{forest_paths}
    \begin{minipage}[b]{0.33\textwidth}
      \centering
        \includegraphics[width=0.98\textwidth]{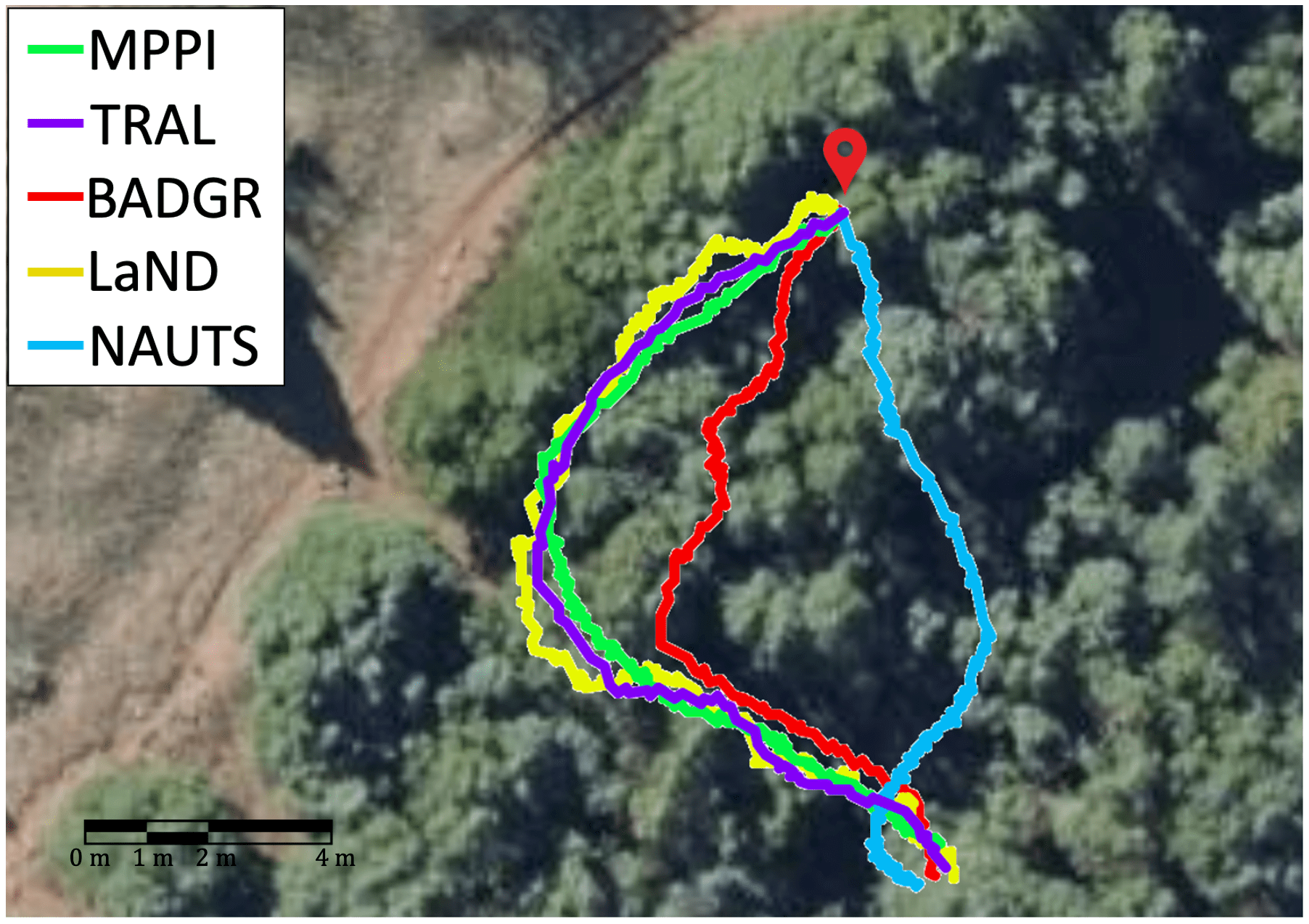}
        \centering
    \end{minipage}}
    \subfigure[Importance of different policies]{
    \label{forest_negotiation}
    \begin{minipage}[b]{0.32\textwidth}
      \centering
        \includegraphics[width=0.9\textwidth]{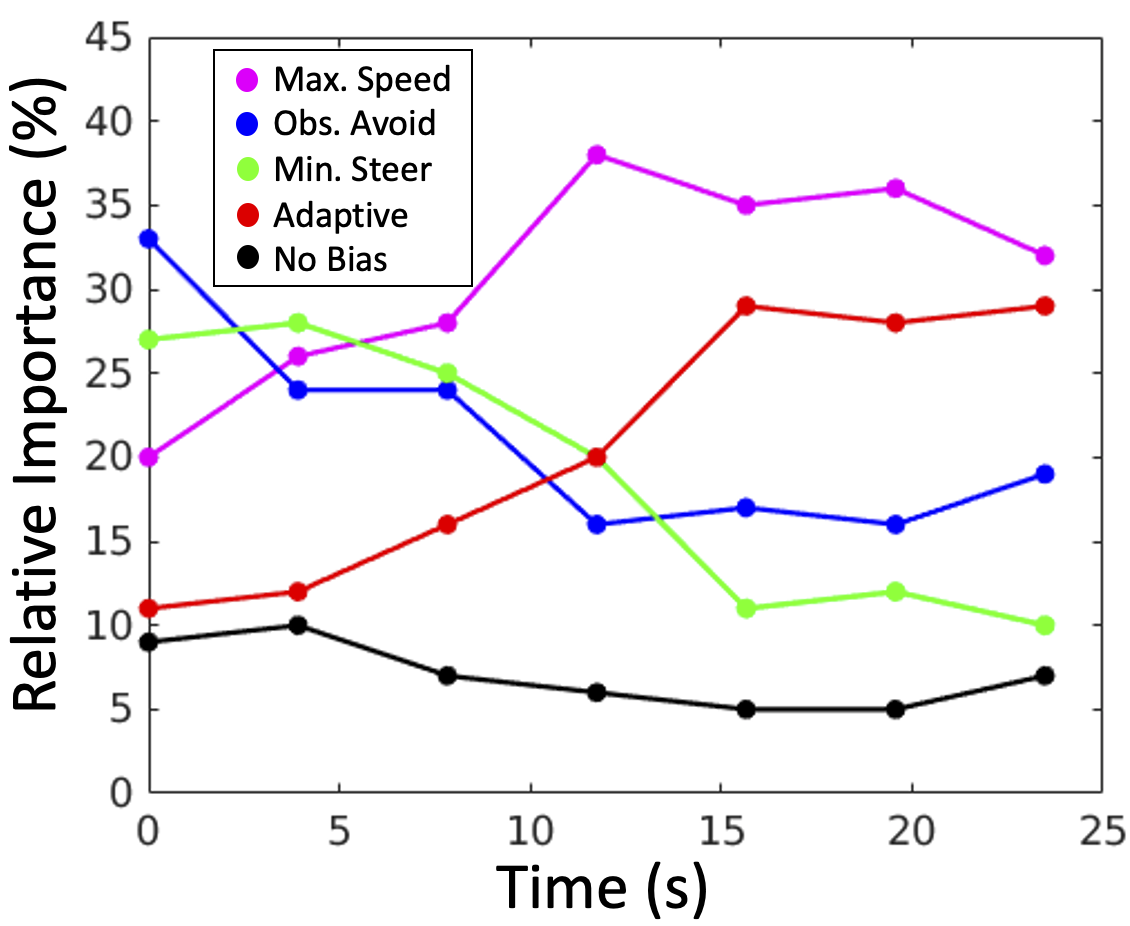}
        \centering
    \end{minipage}
      }\vspace{-8pt}
  \caption{The forest terrain used in our experiments and the qualitative results over this terrain.  
  }  \label{forest_terrain_analysis} 
 \vspace{-6pt}
\end{figure*}

When comparing the traversal time and the distance traversed by the different methods, we observe that NAUTS obtains the best performance followed by BADGR and TRAL. 
The LaND and MPPI approaches obtain higher TT and DT metrics, with MPPI performing the poorest in terms of DT and LaND performing poorest in terms of TT. 
A qualitative comparison, from a single trial, of the path traversed by these methods is provided in Fig. \ref{tall_grass_paths}. Notice, MPPI, LaND, and TRAL all consider tall grass as obstacles and avoid this terrain while traversing. 
We observe that BADGR and NAUTS explore tall grass terrain and the shortest path is taken with our NAUTS approach resulting in the lowest DT and TT values.

NAUTS also performs better than the TRAL and MPPI approaches in terms of the AT metric. The AT metric is observed when robots 
encounter an unseen terrain and require adaptation. In this environment, that happens when the robot transitions into the tall grass terrain.
We do not provide AT values for BADGR and LaND as both approaches have a fixed linear velocity without adaptation. 
Overall, we observe that our approach obtains successful navigation (from FR metric) and better efficiency (from TT and DT metrics) over previous methods. 

Fig. \ref{tall_grass_negotiation} illustrates the NAUTS negotiation process between the five policies in the tall grass terrain. NAUTS learns optimal combinations of policies in real-time during execution (each update is marked by dots in the figure). Initially, max speed has higher importance over other policies.  However, as the robot enters tall grass, obstacle avoidance becomes more important.  While traversing further, the robot learns to give more importance to the max speed policy again and obstacle avoidance becomes less important. All other policies have relatively low importance, but they never reach zero, as NAUTS 
regularly evaluates
the different policies.

\begin{table}[htb]
\centering
\caption{Quantitative results for scenarios when the robot traverses over unseen
dynamic, unstructured off-road forest terrain. 
}
\label{tab:RTT2}
\tabcolsep=0.07cm
\centering
\begin{tabular}{ | c|| c| c| c| c| c|}
\hline
Metrics & MPPI \cite{williams2016aggressive}  & TRAL \cite{siva2019robot} & BADGR \cite{kahn2021badgr}  &  LaND \cite{kahn2021land} & \textbf{NAUTS} \\
\hline
FR (/10) & 5 & 5 & 4 & 7 & \textbf{2} \\
TT (s) & 34.28 & 33.95 & 26.17 & 33.98 & \textbf{24.21} \\
DT (m) & 24.68 & 23.77 & 20.94 & 26.51 & \textbf{16.45} \\
AT (s) & 10.04 & 11.93 & -- & -- & \textbf{7.32} \\
\hline
\end{tabular}
\vspace{-6pt}
\end{table}

\subsection{Navigating on Unseen Unstructured Forest Terrain}
In this set of experiments, we evaluate navigation across forest terrains. Apart from high uncertainty and dynamic obstacles, this terrain has different characteristics that the robot has not previously seen during training, e.g, terrain covered with wood chips, dried leaves, 
rocks, and tree branches.
Similar to the previous set of experiments, the evaluation metrics in the forest terrain are computed across ten runs for each of the methods.

The quantitative results over off-road forest terrain are presented in Table \ref{tab:RTT2}. In terms of the FR metric, we observe a similar trend seen in the tall grass terrain experiments. Specifically, MPPI and TRAL have similar performance in terms of FR metrics. Our NAUTS approach obtains the lowest FR value followed by the BADGR approach, and the LaND approach obtains the highest value. 
Common failures in the forest terrain occur when tree branches occluding the terrain are classified as obstacles or traversing over large rocks, wooden tree barks, or mud in the terrain cause the robot to get stuck. 
NAUTS also obtains better efficiency in both the TT and DT metrics, followed by the BADGR approach. Again, MPPI and TRAL both obtain similar TT and DT values, and LaND exhibits the worst performance. 

Fig. \ref{forest_paths} illustrates qualitatively how MPPI, TRAL, and LaND avoid uncertain and unseen paths and follow an existing trail. However, BADGR explores unknown paths, reaching the goal faster than 
these methods, and NAUTS outperforms all methods by exploring different policies in this unseen terrain.
In this set of experiments, the AT metric is observed throughout navigation as each section of the terrain is not previously seen by the robot and requires the robot to adapt.
NAUTS obtains better AT values than MPPI and TRAL, indicating a better adaptation capability.

Fig. \ref{forest_negotiation} illustrates the negotiation process by NAUTS in the forest terrain.
At the start of the navigation task, each policy has different importance, with obstacle avoidance being the most significant. As the robot continues with the navigation task, it learns to use the optimal combination of policies, which results in the most efficient navigation. Thus, the max speed and adaptive navigational policies become more significant than other policies. 
It is important to note that there is no single optimal policy throughout navigation due to i) the highly unstructured nature of this terrain and ii) the continuous exploration of the NAUTS approach.

\section{Conclusion}\label{sec:conclusion}

In this paper, we introduce the novel NAUTS approach for robot adaptation by 
negotiation for navigating in unstructured terrains, that enables ground robots to adapt their navigation policies using a negotiation process.
Our approach learns a non-linear prediction model to function as a terrain-aware joint local controller and planner corresponding to various policies, and then uses the negotiation process to form agreements between these policies in order to improve robot navigation efficiency.
Moreover, our approach explores different policies to improve navigation efficiency in a given environment continuously.
We also 
developed
a novel optimization algorithm that solves the global optimal solution to the robot negotiation problem in real-time.
Experimental results have shown that our approach enables a robot to negotiate its behaviors with the terrain and delivers more successful and efficient navigation compared to the previous methods.


\bibliographystyle{IEEEtran}
\bibliography{references}

\begin{thebibliography}{10}
\providecommand{\url}[1]{#1}
\csname url@rmstyle\endcsname
\providecommand{\newblock}{\relax}
\providecommand{\bibinfo}[2]{#2}
\providecommand\BIBentrySTDinterwordspacing{\spaceskip=0pt\relax}
\providecommand\BIBentryALTinterwordstretchfactor{4}
\providecommand\BIBentryALTinterwordspacing{\spaceskip=\fontdimen2\font plus
\BIBentryALTinterwordstretchfactor\fontdimen3\font minus
  \fontdimen4\font\relax}
\providecommand\BIBforeignlanguage[2]{{%
\expandafter\ifx\csname l@#1\endcsname\relax
\typeout{** WARNING: IEEEtran.bst: No hyphenation pattern has been}%
\typeout{** loaded for the language `#1'. Using the pattern for}%
\typeout{** the default language instead.}%
\else
\language=\csname l@#1\endcsname
\fi
#2}}

\bibitem{lattanzi2017review}
D.~Lattanzi and G.~Miller, ``{Review of Robotic Infrastructure Inspection
  Systems},'' \emph{JIS}, vol.~23, no.~3, p. 04017004, 2017.

\bibitem{schuster2019towards}
M.~J. Schuster, S.~G. Brunner, K.~Bussmann, S.~B{\"u}ttner, A.~D{\"o}mel,
  M.~Hellerer, H.~Lehner, P.~Lehner, Porges, \emph{et~al.}, ``{Towards
  Autonomous Planetary Exploration},'' \emph{JINT}, vol.~93, no.~3, pp.
  461--494, 2019.

\bibitem{chiang2020safety}
H.-T.~L. Chiang, B.~HomChaudhuri, L.~Smith, and L.~Tapia, ``{Safety,
  Challenges, and Performance of Motion Planners in Dynamic Environments},'' in
  \emph{Robotics Research}.\hskip 1em plus 0.5em minus 0.4em\relax Springer,
  2020, pp. 793--808.

\bibitem{sartoretti2018central}
G.~Sartoretti, S.~Shaw, K.~Lam, N.~Fan, M.~Travers, and H.~Choset, ``{Central
  Pattern Generator with Inertial Feedback for Stable Locomotion and Climbing
  in Unstructured Terrain},'' in \emph{ICRA}, 2018.

\bibitem{siva2019robot}
S.~Siva, M.~Wigness, J.~Rogers, and H.~Zhang, ``{Robot Adaptation to
  Unstructured Terrains by Joint Representation and Apprenticeship Learning},''
  in \emph{RSS}, 2019.

\bibitem{kavraki1996probabilistic}
L.~E. Kavraki, P.~Svestka, J.-C. Latombe, and M.~H. Overmars, ``{Probabilistic
  Roadmaps for Path Planning in High-Dimensional Configuration Spaces},''
  \emph{T-RO}, vol.~12, no.~4, pp. 566--580, 1996.

\bibitem{williams2016aggressive}
G.~Williams, P.~Drews, B.~Goldfain, J.~M. Rehg, and E.~A. Theodorou,
  ``{Aggressive Driving with Model Predictive Path Integral Control},'' in
  \emph{ICRA}, 2016.

\bibitem{moysis2020chaotic}
L.~Moysis, E.~Petavratzis, C.~Volos, H.~Nistazakis, and I.~Stouboulos, ``{A
  Chaotic Path Planning Generator Based on Logistic Map and Modulo Tactics},''
  \emph{RAS}, vol. 124, p. 103377, 2020.

\bibitem{kahn2021land}
G.~Kahn, P.~Abbeel, and S.~Levine, ``{LaND: Learning to Navigate from
  Disengagements},'' \emph{RAL}, vol.~6, no.~2, pp. 1872--1879, 2021.

\bibitem{kahn2021badgr}
P.~A. Gregory~Kahn and S.~Levine, ``{BADGR: An Autonomous Self-Supervised
  Learning-based Navigation System},'' vol.~6, no.~2, 2021, pp. 1312--1319.

\bibitem{kahn2018self}
G.~Kahn, A.~Villaflor, B.~Ding, P.~Abbeel, and S.~Levine, ``{Self-supervised
  Deep Reinforcement Learning with Generalized Computation Graphs for Robot
  Navigation},'' in \emph{ICRA}, 2018.

\bibitem{han2018sensor}
S.-H. Han, H.-J. Choi, P.~Benz, and J.~Loaiciga, ``{Sensor-Based Mobile Robot
  Navigation via Deep Reinforcement Learning},'' in \emph{BIGCOMP}, 2018.

\bibitem{kumar2021rma}
A.~Kumar, Z.~Fu, D.~Pathak, and J.~Malik, ``{RMA: Rapid Motor Adaptation for
  Legged Robots},'' \emph{RSS}, 2021.

\bibitem{liu2021lifelong}
B.~Liu, X.~Xiao, and P.~Stone, ``{A Lifelong Learning Approach to Mobile Robot
  Navigation},'' \emph{RAL}, 2021.

\bibitem{zenke2017continual}
F.~Zenke, B.~Poole, and S.~Ganguli, ``{Continual Learning through Synaptic
  Intelligence},'' in \emph{ICML}, 2017.

\bibitem{kahn2018composable}
G.~Kahn, A.~Villaflor, P.~Abbeel, and S.~Levine, ``{Composable
  Action-Conditioned Predictors: Flexible Off-policy Learning for Robot
  Navigation},'' in \emph{CoRL}, 2018.

\bibitem{jhang2018navigation}
J.-Y. Jhang, C.-J. Lin, C.-T. Lin, and K.-Y. Young, ``{Navigation Control of
  Mobile Robots Using an Interval Type-2 Fuzzy Controller Based on
  Dynamic-group Particle Swarm Optimization},'' \emph{IJCAS}, vol.~16, no.~5,
  pp. 2446--2457, 2018.

\bibitem{sinha2020neural}
A.~Sinha, M.~O'Kelly, R.~Tedrake, and J.~C. Duchi, ``{Neural Bridge Sampling
  for Evaluating Safety-Critical Autonomous Systems},'' \emph{NIPS}, 2020.

\bibitem{sinha2020formulazero}
A.~Sinha, M.~O’Kelly, H.~Zheng, R.~Mangharam, J.~Duchi, and R.~Tedrake,
  ``{Formulazero: Distributionally Robust Online Adaptation via Offline
  Population Synthesis},'' in \emph{ICML}, 2020.

\bibitem{saffiotti1997uses}
A.~Saffiotti, ``{The uses of Fuzzy Logic in Autonomous Robot Navigation},''
  \emph{IJSC}, vol.~1, no.~4, pp. 180--197, 1997.

\bibitem{wang2008fuzzy}
M.~Wang and J.~N. Liu, ``{Fuzzy Logic-Based Real-Time Robot Navigation in
  Unknown Environment with Dead Ends},'' \emph{RAS}, vol.~56, no.~7, pp.
  625--643, 2008.

\bibitem{rabiner1978fir}
L.~Rabiner, R.~Crochiere, and J.~Allen, ``{FIR System Modeling and
  Identification in the Presence of Noise and with Band-Limited Inputs},''
  \emph{ICASSP}, vol.~26, no.~4, pp. 319--333, 1978.

\bibitem{bolea2003non}
Y.~Bolea, A.~Grau, and A.~Sanfeliu, ``{Non-speech Sound Feature Extraction
  Based on Model Identification for Robot Navigation},'' in \emph{CIARP}, 2003.

\bibitem{pebrianti2018motion}
D.~Pebrianti, Y.~H. Hao, N.~A.~S. Suarin, L.~Bayuaji, Z.~Musa, M.~Syafrullah,
  and I.~Riyanto, ``{Motion Tracker Based Wheeled Mobile Robot System
  Identification and Controller Design},'' in \emph{Intelligent Manufacturing
  \& Mechatronics}, 2018.

\bibitem{van2017motion}
J.~Van Den~Berg, S.~Patil, and R.~Alterovitz, ``{Motion Planning under
  Uncertainty using differential Dynamic Programming in Belief Space},'' in
  \emph{Robotics Research}.\hskip 1em plus 0.5em minus 0.4em\relax Springer,
  2017, pp. 473--490.

\bibitem{zhang2012iterative}
H.-j. Zhang, J.-w. Gong, Y.~Jiang, G.-m. Xiong, and H.-y. Chen, ``{An Iterative
  Linear Quadratic Regulator based Trajectory Tracking Controller for Wheeled
  Mobile Robot},'' \emph{JZUS-C}, vol.~13, no.~8, pp. 593--600, 2012.

\bibitem{howard2010receding}
T.~M. Howard, C.~J. Green, and A.~Kelly, ``{Receding Horizon Model-Predictive
  Control for Mobile Robot Navigation of Intricate Paths},'' in \emph{FSR},
  2010.

\bibitem{hafez2019integrity}
O.~A. Hafez, G.~D. Arana, and M.~Spenko, ``{Integrity Risk-Based Model
  Predictive Control for Mobile Robots},'' in \emph{ICRA}, 2019.

\bibitem{tahirovic2010general}
A.~Tahirovic and G.~Magnani, ``{General Framework for Mobile Robot Navigation
  using Passivity-based MPC},'' \emph{TACON}, vol.~56, no.~1, pp. 184--190,
  2010.

\bibitem{koopman1931hamiltonian}
B.~O. Koopman, ``{Hamiltonian Systems and Transformation in Hilbert Space},''
  \emph{PNAS}, vol.~17, no.~5, pp. 315--318, 1931.

\bibitem{proctor2018generalizing}
J.~L. Proctor, S.~L. Brunton, and J.~N. Kutz, ``{Generalizing Koopman Theory to
  Allow for Inputs and Control},'' \emph{SIADS}, vol.~17, no.~1, pp. 909--930,
  2018.

\bibitem{williams2015data}
M.~O. Williams, I.~G. Kevrekidis, and C.~W. Rowley, ``{A Data-Driven
  Approximation of the Koopman Operator: Extending Dynamic Mode
  Decomposition},'' \emph{JNS}, vol.~25, no.~6, pp. 1307--1346, 2015.

\bibitem{atkeson1997robot}
C.~G. Atkeson and S.~Schaal, ``{Robot Learning from Demonstration},'' in
  \emph{ICML}, 1997.

\bibitem{wigness2018robot}
M.~Wigness, J.~G. Rogers, and L.~E. Navarro-Serment, ``{Robot Navigation from
  Human Demonstration: Learning Control Behaviors},'' in \emph{ICRA}, 2018.

\bibitem{siva2021enhancing}
S.~Siva, M.~Wigness, J.~Rogers, and H.~Zhang, ``{Enhancing Consistent Ground
  Maneuverability by Robot Adaptation to Complex Off-Road Terrains},'' in
  \emph{CoRL}, 2021.

\bibitem{wang2021appli}
Z.~Wang, X.~Xiao, B.~Liu, G.~Warnell, and P.~Stone, ``{APPLI: Adaptive Planner
  Parameter Learning from Interventions},'' in \emph{ICRA}, 2021.

\bibitem{nampoothiri2021recent}
M.~H. Nampoothiri, B.~Vinayakumar, Y.~Sunny, and R.~Antony, ``{Recent
  Developments in Terrain Identification, Classification, Parameter Estimation
  for the Navigation of Autonomous Robots},'' \emph{SNAS}, vol.~3, no.~4, pp.
  1--14, 2021.

\bibitem{serra2018overcoming}
J.~Serra, D.~Suris, M.~Miron, and A.~Karatzoglou, ``{Overcoming Catastrophic
  Forgetting with Hard Attention to the Task},'' in \emph{ICML}, 2018.

\bibitem{duriez2017machine}
T.~Duriez, S.~L. Brunton, and B.~R. Noack, \emph{{Machine Learning
  Control-Taming Nonlinear Dynamics and Turbulence}}.\hskip 1em plus 0.5em
  minus 0.4em\relax Springer, 2017.

\bibitem{brunton2019data}
S.~L. Brunton and J.~N. Kutz, \emph{{Data-Driven Science and Engineering:
  Machine Learning, Dynamical Systems, and Control}}.\hskip 1em plus 0.5em
  minus 0.4em\relax Cambridge University Press, 2019.

\bibitem{schmid2010dynamic}
P.~J. Schmid, ``{Dynamic Mode Decomposition of Numerical and Experimental
  Data},'' \emph{JFM}, vol. 656, pp. 5--28, 2010.

\bibitem{mamakoukas2019local}
G.~Mamakoukas, M.~Castano, X.~Tan, and T.~Murphey, ``{Local Koopman Operators
  for Data-Driven Control of Robotic Systems},'' in \emph{RSS}, 2019.

\bibitem{wang2021real}
H.~Wang and N.~Noguchi, ``{Real-time States Estimation of a Farm Tractor using
  Dynamic Mode Decomposition},'' \emph{GPS Solutions}, vol.~25, no.~1, pp.
  1--12, 2021.

\bibitem{kutz2016dynamic}
J.~N. Kutz, S.~L. Brunton, B.~W. Brunton, and J.~L. Proctor, \emph{{Dynamic
  Mode Decomposition: Data-driven Modeling of Complex Systems}}.\hskip 1em plus
  0.5em minus 0.4em\relax SIAM, 2016.

\bibitem{caceres2017approach}
C.~C{\'a}ceres, J.~M. Ros{\'a}rio, and D.~Amaya, ``{Approach of Kinematic
  Control for a Non-Holonomic Wheeled Robot using Artificial Neural Networks
  and Genetic Algorithms},'' in \emph{IWOBI}, 2017.

\bibitem{ramirez1999nonlinear}
D.~R. Ram{\'\i}rez, D.~Lim{\'o}n, J.~Gomez-Ortega, and E.~F. Camacho,
  ``{Nonlinear MBPC for mobile robot navigation using genetic algorithms},'' in
  \emph{ICRA}, 1999.

\bibitem{gillespie2018learning}
M.~T. Gillespie, C.~M. Best, E.~C. Townsend, D.~Wingate, and M.~D. Killpack,
  ``{Learning Nonlinear Dynamic Models of Soft Robots for Model Predictive
  Control with Neural Networks},'' in \emph{RoboSoft}, 2018.

\bibitem{nagariya2020iterative}
A.~Nagariya and S.~Saripalli, ``{An Iterative LQR Controller for Off-road and
  On-road Vehicles using a Neural Network Dynamics Model},'' in \emph{IV},
  2020, pp. 1740--1745.

\bibitem{alharbi2020global}
M.~Alharbi and H.~A. Karimi, ``{A Global Path Planner for Safe Navigation of
  Autonomous Vehicles in Uncertain Environments},'' \emph{Sensors}, vol.~20,
  no.~21, p. 6103, 2020.

\bibitem{garriga2018deep}
A.~Garriga-Alonso, C.~E. Rasmussen, and L.~Aitchison, ``{Deep Convolutional
  Networks as Shallow Gaussian Processes},'' \emph{ICLR}, 2019.

\bibitem{zhang2018generalized}
Z.~Zhang and M.~R. Sabuncu, ``{Generalized Cross Entropy Loss for Training Deep
  Neural Networks with Noisy Labels},'' in \emph{NIPS}, 2018.

\bibitem{chan2019assistive}
L.~Chan, D.~Hadfield-Menell, S.~Srinivasa, and A.~Dragan, ``{The Assistive
  Multi-Armed Bandit},'' in \emph{HRI}, 2019.

\bibitem{balasubramanian2021zeroth}
K.~Balasubramanian and S.~Ghadimi, ``{Zeroth-Order Nonconvex Stochastic
  Optimization: Handling Constraints, High Dimensionality, and Saddle
  Points},'' \emph{FoCM}, vol.~22, no.~1, pp. 35--76, 2022.

\bibitem{rigatos2010extended}
G.~G. Rigatos, ``{Extended Kalman and Particle Filtering for Sensor Fusion in
  Motion Control of Mobile Robots},'' \emph{IMACS}, vol.~81, no.~3, pp.
  590--607, 2010.

\end{thebibliography}

\begin{center}
\textbf{\Large NAUTS: Negotiation for Adaptation to Unstructured Terrain Surfaces} 
\end{center}
\begin{center}
\textbf{\Large \emph{Supplementary Material}}
\end{center}
\vspace{10px}
\setcounter{equation}{0}
\setcounter{section}{0}
\setcounter{figure}{0}
\setcounter{table}{0}
\setcounter{page}{1}
\makeatletter
\renewcommand{\theequation}{\arabic{equation}}
\renewcommand{\thefigure}{\arabic{figure}}
In this supplementary material document, Section \ref{sec:proof} presents the proof of convergence for the optimization algorithm proposed in the main paper and section \ref{sec:experimental_details} discusses the additional details on our experimentation procedure.  

\section{Proof of Convergence for the Optimization Algorithm}\label{sec:proof}

In the following, we prove that Algorithm 1 (in the main paper) decreases the value of the objective function in Eq. (4) (of the main paper) with each iteration during execution and converges to the global optimal solution. 

At first, we present a lemma:
\begin{lemma}\label{lemma1}
For any two given vectors $\mathbf{a}$ and $\mathbf{b}$, the following inequality relation holds:
$\|\mathbf{b}\|_2 - \frac{\|\mathbf{b}\|_2^2}{2\|\mathbf{a}\|_2}
\leq
\|\mathbf{a}\|_2 - \frac{\|\mathbf{a}\|_2^2}{2\|\mathbf{a}\|_2}$
\end{lemma}
\begin{proof}
\begin{equation}
-(\Vert\mathbf{b}\Vert_{2}-\Vert\mathbf{a}\Vert_{2})^2 \leq 0  \nonumber
\end{equation}
\begin{equation}
-\Vert\mathbf{b}\Vert_{2}^{2} - \Vert\mathbf{a}\Vert_{2}^{2} + 		        2\Vert\mathbf{b}\Vert_{2}\Vert\mathbf{a}\Vert_{2} \leq 0   \nonumber
\end{equation}
\begin{equation}
2\Vert\mathbf{b}\Vert_{2}\Vert\textbf{a}\Vert_{2} - \Vert\mathbf{b}\Vert_{2}^{2} \leq \Vert\textbf{a}\Vert_{2}^{2}       \nonumber
\end{equation}	
\begin{equation}
\Vert\mathbf{b}\Vert_{2} - \dfrac{\Vert\mathbf{b}\Vert_{2}^{2}}{2\Vert\textbf{a}\Vert_{2}} \leq  \Vert\textbf{a}\Vert_{2} - \dfrac{\Vert\textbf{a}\Vert_{2}^{2}}{2\Vert\textbf{a}\Vert_{2}}  \nonumber
\end{equation}
\end{proof}
From Lemma \ref{lemma1}, we can derive the following corollary:
\begin{corollary}\label{corollary1}
For any two given matrices $\mathbf{A}$ and $\mathbf{B}$ , the following inequality relation holds:
 \begin{eqnarray}
\Vert\mathbf{B}\Vert_E - \frac{\Vert\mathbf{B}\Vert_E^2}{2\Vert\mathbf{A}\Vert_E}
\leq
\Vert\mathbf{A}\Vert_E - \frac{\Vert\mathbf{A}\Vert_E^2}{2\Vert\mathbf{A}\Vert_E} \nonumber
\end{eqnarray}
\end{corollary}
where the operator $\Vert \cdot \Vert_{E}$ is the exploration norm introduced in the main paper.

\begin{theorem}\label{thm2}
Algorithm 1 (in the main paper) converges fast to the global optimal solution to the terrain negotiation problem in Eq. (4) (in the main paper) during execution.
\end{theorem}

\begin{proof}
According to Step 8 of Algorithm 1, for each iteration step $s$ during optimization, the value of $\mathbf{v}^{i}(s+1)$ can be given as:
\begin{align}\label{eq:proof1_1}
&\mathbf{v}^{i}(s+1)=\Vert r^{*}(s+1) - (\mathbf{o}_t^i)^{\top} \mathbf{v}^{i*}(s+1){r^{i}(s+1)} \Vert^{2}_{2} \\
& \quad \quad +\sum_{i=1}^{N}(\lambda_4 (\mathbf{v}^{i}(s+1))^\top\mathbf{Q}(s+1)(\mathbf{v}^{i}(s+1))) \nonumber
\end{align}
where $\mathbf{Q}(s+1) = \frac{\mathbf{I}}{2\Vert \mathbf{V}(s)\Vert_{E}} $.
Then we derive that:
\begin{eqnarray}\label{eq:proof2}
&& \mathcal{J}(s+1) + \sum_{i=1}^{N}(\lambda_4 (\mathbf{v}^{i}(s+1))^\top\mathbf{Q}(s+1)(\mathbf{v}^{i}(s+1))) \nonumber \\
&& \leq  \mathcal{J}(s) + \sum_{i=1}^{N}(\lambda_4 (\mathbf{v}^{i}(s))^\top\mathbf{Q}(s)(\mathbf{v}^{i}(s))) 
\end{eqnarray}
where $\mathcal{J}(s)=\Vert r^{*}(s) - (\mathbf{o}_t^i)^{\top}\mathbf{v}^{i*}(s) {r^{i}(s)} \Vert^{2}_{2}$.

After substituting the definition $\mathbf{Q}$ in Eq. (\ref{eq:proof2}), we obtain
\begin{eqnarray}\label{eq:proof3}
&&\mathcal{J}(s+1)+ (\lambda_4\dfrac{\|\mathbf{V}(s+1)\|_E^2}{2\|\mathbf{V}(s)\|_E}) \nonumber\\
&&\leq \mathcal{J}(s)+ (\lambda_4\dfrac{\|\mathbf{V}(s)\|_E^2}{2\|\mathbf{V}(s)\|_E})
\end{eqnarray}
From Corollary \ref{corollary1}, for the weight matrix $\mathbf{V}$ we have: 
\begin{eqnarray}\label{eq:proof4}
&&\Bigg({\|\mathbf{V}(s+1)\|_E} - {\dfrac{\|\mathbf{V}(s+1)\|_E^2}{2\|\mathbf{V}(s)\|_E}}\Bigg)\nonumber\\ &&\leq \Bigg({\|\mathbf{V}(s)\|_E} - {\dfrac{\|\mathbf{V}(s)\|_E^2}{2\|\mathbf{V}(s)\|_E}}\Bigg).
\end{eqnarray}
Adding Eq. (\ref{eq:proof3}) and (\ref{eq:proof4}) on both sides, we have
\begin{eqnarray}\label{eq:proof6}
&&\mathcal{J}(s+1)+ \lambda_4{\|\mathbf{V}(s+1)\|_E}\nonumber\\
&&\leq \mathcal{J}(s)+\lambda_4{\|\mathbf{V}(s)\|_E}
\end{eqnarray}
Eq. (\ref{eq:proof6}) implies that the updated value of weight matrix $\mathbf{V}$, decreases the value of the objective function with each iteration.
As the negotiation problem in Eq. (4) (in the main paper) is convex, Algorithm 1 (in the main paper) converges to the global optimal solution. 
Furthermore, during each time step of execution, we start with near-optimal $\mathbf{V}$ from previous time steps and as the objective is convex, Algorithm 1 converges faster than when starting from initial conditions, i.e., $\mathbf{V}$ as a zero matrix. 

\end{proof}

\section{Experimental Details}\label{sec:experimental_details}

We use a Clearpath Husky ground robot for our field experiments to demonstrate the negotiation capability during terrain navigation. In addition to the Intel Realsense D435 color camera, an Ouster OS1-64 LiDAR, GPS, Microstrain 3DM-GX5-25 IMU, and wheel odometers, the robot is also equipped with a 4.3 GHz i7 CPU with 16GB RAM and Nvidia 1660Ti GPU. 

For collecting the training data, a human expert demonstrates robot driving over simple terrains of short grass, medium-sized rocks, large-sized rocks, gravels, and concrete while following one of the following five navigational policies:
\begin{itemize}
    \item \emph{Maximum Speed:} When following this navigational policy, the human expert drives the husky robot with the maximum traversal speed irrespective of the terrain. 
    \item \emph{Obstacle Avoidance:} While following this policy, the expert demonstrates the robot to maneuver by driving around the obstacles and avoids collision. 
    \item \emph{Minimum Steering:} For this policy, the expert drives the robot with limited steering. The linear velocity is fixed during navigation. To perform obstacle avoidance, the robot turns from farther distances instead of making short and acute turns.
    \item \emph{Adaptive Maneuvers:}  While following this policy, the expert varies the robot's speed with each terrain to reduce the jerkiness of the robot. Specifically, the expert uses high speeds maneuvers in short-grass and concrete terrains but slower speeds in the terrains of medium rocks and gravels and the slowest in the terrain of large rocks. 
    \item \emph{No Navigational Bias:} When following this policy, the expert demonstrates navigation in various scenarios without particular policy bias, i.e., without following particular navigational policies.
\end{itemize}

For each policy, the robot is driven on all five terrains for an hour, which nearly equals 108000 distinctive terrain observations and the corresponding sequence of robot navigational behaviors and states for each navigational policy.
No further pre-processing is performed on the collected data. We use this data to learn optimal $\pi^{i}$, $i=1,\dots, N$ and $\mathbf{V}$ during training. We learn these parameters for different value of hyper-parameters to NAUTS approach, i.e., $\lambda_{1}$, $\lambda_{2}$, $\lambda_{3}$, $\lambda_{4}$ and $T$. The combination of these hyper-parameters that results in the best performance of NAUTS during testing are then used for our experiments. In our case, the optimal performance of NAUTS is obtained at  
 $\lambda_{1}=0.1, \lambda_{2}=10$, $\lambda_{3}=1$ and $\lambda_{4}=0.1$ for $T=9$.

\end{document}